\documentclass[letterpaper]{article}
\usepackage{times} % DO NOT CHANGE THIS
\usepackage{helvet} % DO NOT CHANGE THIS
\usepackage{courier} % DO NOT CHANGE THIS
\usepackage[hyphens]{url} % DO NOT CHANGE THIS
\usepackage{graphicx} % DO NOT CHANGE THIS
\usepackage{graphicx} % DO NOT CHANGE THIS
\usepackage{caption} % DO NOT CHANGE THIS
\DeclareCaptionStyle{ruled}%
{labelfont=normalfont,labelsep=colon,strut=off}
\usepackage{amsfonts,amssymb,amsmath,amsthm} % usual
\usepackage{tikz-cd}
\usepackage{color}
\usepackage[toc,page]{appendix}
\usepackage{mathtools}
\usepackage{neuralnetwork}
\usepackage{empheq}
\usepackage{subfig}
\usepackage{authblk}
\usepackage{breqn}
\usepackage{tcolorbox}
\newtheorem{theorem}{Theorem}[section]
\newtheorem{lemma}[theorem]{Lemma}

\newtheorem{proposition}[theorem]{Proposition}

\theoremstyle{definition}
\newtheorem{definition}[theorem]{Definition}
\newtheorem{Corollary}[theorem]{Corollary}
\newtheorem{remark}[theorem]{Remark}
\newtheorem{notation}[theorem]{Notation}
\newtheorem{example}[theorem]{Example}

\title{On the space of coefficients of a Feed Forward Neural Network}
\author{Dinesh Valluri\footnote{dvalluri@uwo.ca} }
\author{Rory Campbell\footnote{rcamp@uwo.ca}}
\date{}

\affil{Department of Computer Science, The Univerity of Western Ontario}

\begin{document}

\maketitle

\begin{abstract}
We define and establish the conditions for `equivalent neural networks' - neural networks with different weights, biases, and threshold functions which result in the same associated function. We prove that, given a neural network $\mathcal{N}$ with piece-wise linear activation, the space of coefficients describing all equivalent neural networks is given by a semialgebraic set. This result is obtained by studying different representations of a given piece-wise linear function using the Tarski-Seidenberg theorem. 

\iffalse
This result is obtained, in part, by studying the `relevant' and `redundant' terms in piece-wise linear functions.
\fi
\end{abstract}

\section{Introduction}
Neural networks are a critical component of AI research and industry. They are a prominent part of modern research in self-driving cars \cite{bojarski2016end}, fraud detection \cite{zakaryazad2016profit}, medical diagnostics \cite{amato2013artificial}, and recommendation systems for popular destinations on the web \cite{covington2016deep}. The prevalence of these techniques necessitates a stronger understanding of their inner workings to better harness their potential and ensure their correct use.

The power of neural networks and the desire for an understanding of these tools have fostered a field of study surrounding `explainable AI'. Explainable AI is useful for examining ML results in a scientific context \cite{roscher2020explainable} and for explaining ML results to stakeholders in enterprise contexts \cite{bhatt2020explainable}. Machine learning engineers looking to debug their work also benefit from research into explainable ML, which has been a partial motivation behind the development of tools such as LIME \cite{ribeiro2016should}.

In addition to creating tools for explainable AI, there has been an increase in research into understanding the mathematics of neural networks. One direction of this research is examining neural networks with piecewise linear activation functions \cite{arora2016understanding}. The other being Zhang et. al \cite{zhang2018tropical}, which draws an explicit connection between tropical rational functions and feed-forward neural networks with piecewise linear activation functions.

In this paper we study the space of coefficients of a feedforward neural network with piece-wise linear activation functions. The key idea is that one might have different weights, biases, and thresholds for a neural architecture, yet resulting in the same associated function. This leaves the question of characterizing the space of coefficients of a neural network whose associated function is fixed. We answer this question by deducing that the space of such coefficients is given by a semialgebraic subset of $\mathbb{R}^{N}$, for some $N$. %An algorithmic approach to computing such semialgebraic sets is presented and demonstrated with the example of an Autoencoder.
 
In section 2 we begin by recalling some basic aspects of semialgebraic sets necessary for the main result. This consists of a version of the Tarski-Seidenberg theorem and an application relevant to the key theorem. In section 3, we define the notion of a feed-forward neural network. Unlike in some other sources what we mean by a \emph{neural network} is the data of weights, biases, and threshold vectors on a neural architecture i.e., a directed acyclic graph. We associate a function to the neural network by composing the activation functions as done in say \cite{zhang2018tropical}. We say that two neural networks are equivalent if their associated functions are identical. In section 4, we study the algebraic nature of these equivalence classes. In particular, using the Tarski-Seidenberg theorem we deduce that each equivalence class is given by a semialgebraic set. %We remark that the polynomials defining the semialgebraic relations can be computed using the algorithm in section 2. 

\section*{Acknowledgements}
The authors would like to thank Dr.~Mark Daley for pointing towards the literature which inspired this research. This work builds on notions developed in Zhang et. al \cite{zhang2018tropical}. This work was done in affiliation with the Computational Convergence Lab at The University of Western Ontario; we extend our thanks to all its members.

\section{Semialgebraic Sets}

A \emph{semialgebraic} set is a subset of $\mathbb{R}^n$ satisfying a finite number of polynomial equations and inequations with coefficients in $\mathbb{R}$ \cite{coste2000introduction}. Note that unlike algebraic sets the class of semialgebraic sets form a Boolean algebra, i.e., they are closed under union, intersection and complement in $\mathbb{R}^n$. Now we build the set-up necessary to state a version of Tarski-Seidenberg theorem \cite{mishra1993algorithmic}.

\begin{definition}
A \emph{first-order formula} is constructed in the following manner, as defined in \cite{coste2000introduction}:
\begin{enumerate}
    \item If $P\in\mathbb{R}[X_{1},...,X_{n}]$ then $P=0$ and $P>0$ are first-order formulae.
    \item If $\Phi$ and $\Psi$ are first-order formulae, then ``$\Phi$ and $\Psi$", ``$\Phi$ or $\Psi$", ``not $\Phi$" (often denoted by $\Phi\land\Psi$, $\Phi\lor\Psi$ and $\neg\Phi$, respectively) are first-order formulae.
    \item If $\Phi$ is a formula and $X$, a variable ranging over $\mathbb{R}$, then $\exists X\Phi$ and $\forall X\Phi$ are first-order formulae.
\end{enumerate}
\end{definition}

The following version of Tarski-Seidenberg theorem allows us to eliminate quantifiers from first-order formulae. We apply this theorem to a first-order formula involving linear inequalities. These inequalities arise naturally in our study of feed forward neural networks with ReLu activation as shown in section 4.  

\begin{figure*}
\centering
    \subfloat[\centering label 1]{{\includegraphics[width=5cm]{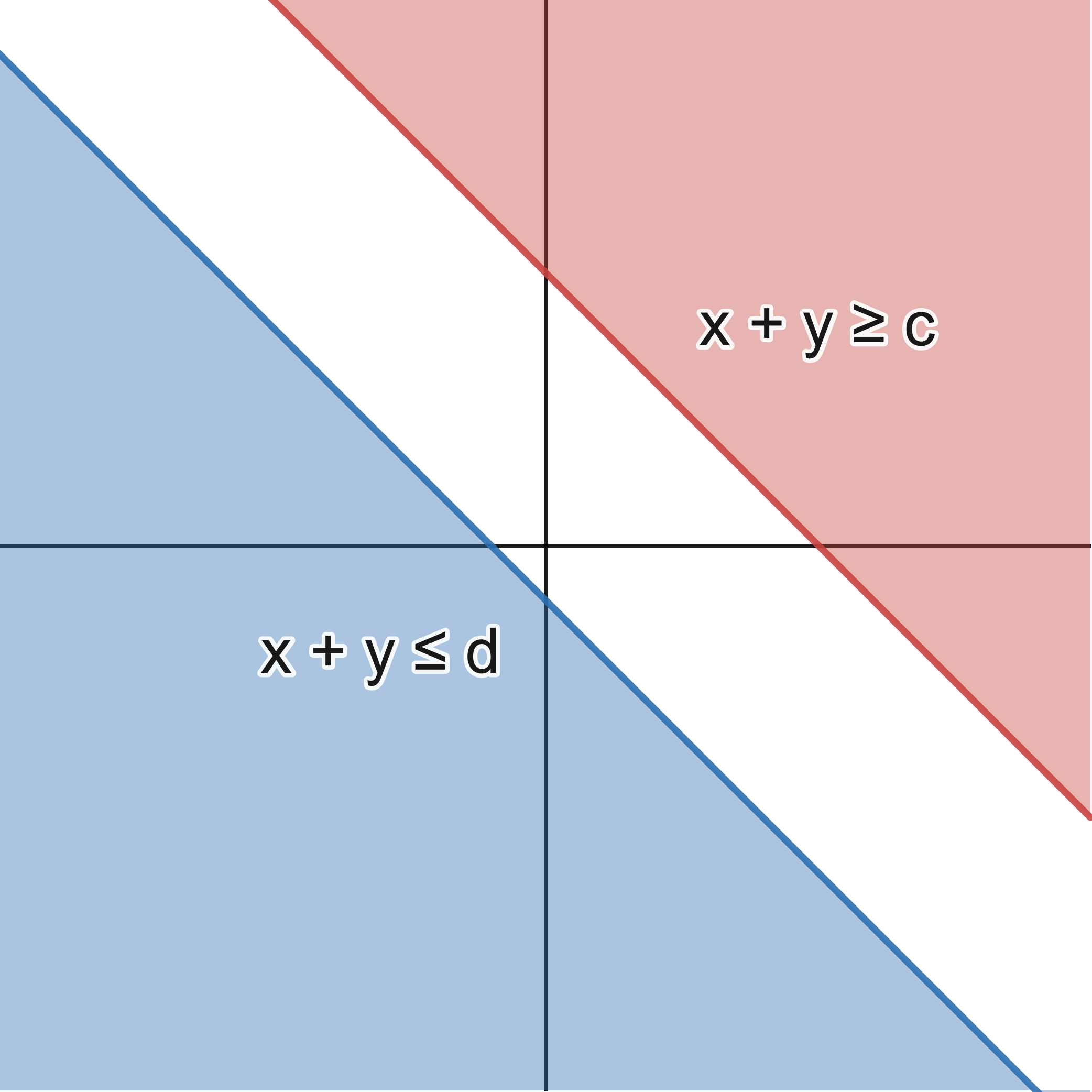} }}%
    \qquad
    \subfloat[\centering label 2]{{\includegraphics[width=5cm]{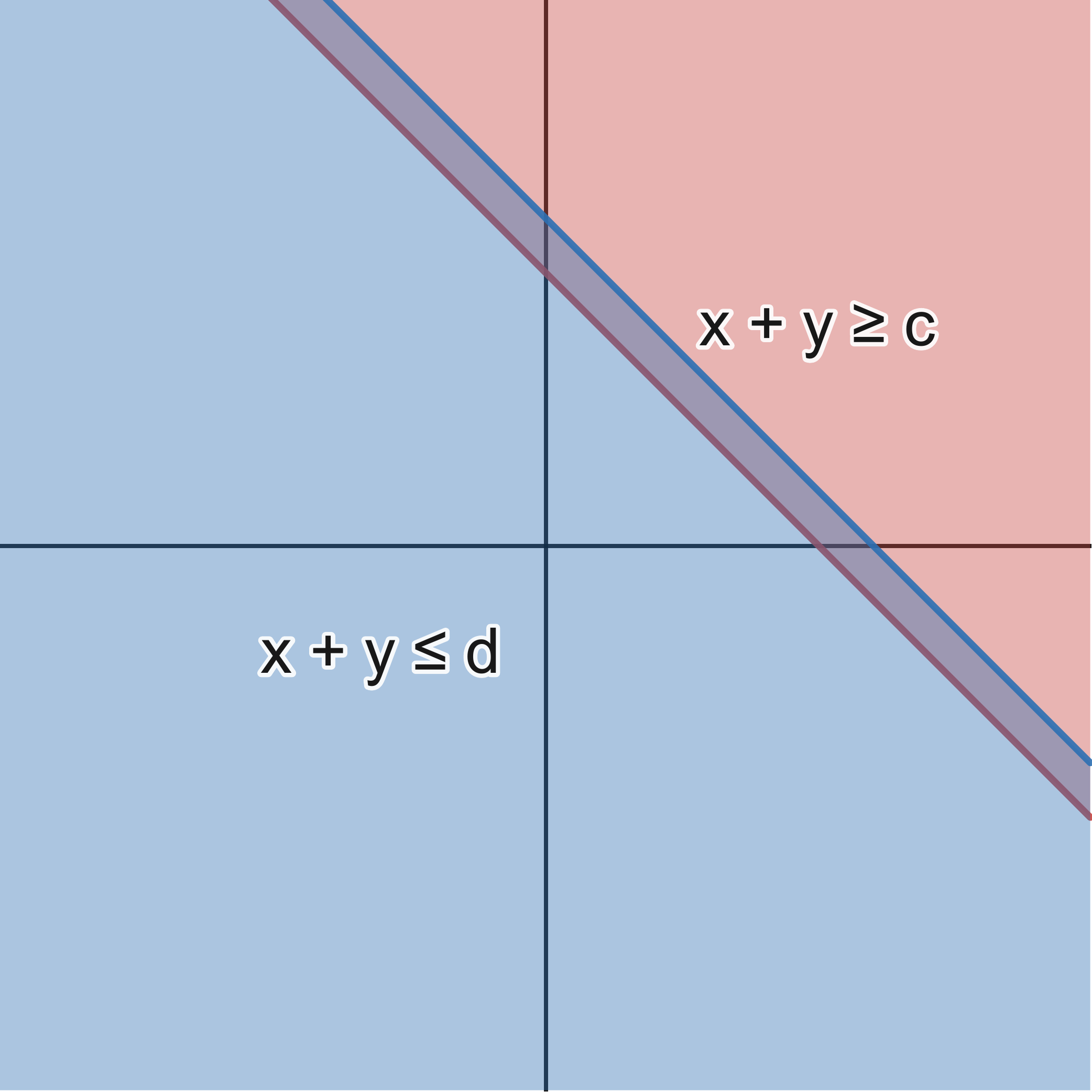} }}%
    \caption{Visual for Example 2.4}%
    \label{fig:example}%
\end{figure*}

\begin{theorem}[Tarski-Seidenberg Theorem]
If $\Phi(X_{1},...,X_{n})$ is a first-order formula, the set of $(x_{1},...,x_{n})\in\mathbb{R}^{n}$ which satisfy $\Phi(x_{1},...,x_{n})$ is semialgebraic.
\end{theorem}

\iffalse
A fundamental observation about semi algebraic sets is the Tarski-Seidenberg theorem, which states that...  In this paper we use Tarski-Seidenberg theorem and a simple algorithm in its spirit to eliminate quantifiers, presented in \cite{coste2000introduction} as follows

\begin{theorem}[Tarski-Seidenberg Theorem]
Let A be a semialgebraic subset of $\mathbb{R}^{n+1}$ and $\Pi: \mathbb{R}^{n+1} \rightarrow \mathbb{R}^{n}$ , the projection on the first $n$ coordinates. Then $\Pi(A)$ is a semialgebraic subset of $\mathbb{R}^{n}$ which satisfy $\Phi($
\end{theorem}

  The Tarski-Seidenberg theorem can be phrased as a quantifier elimination problem, as in \cite{coste2000introduction}, by considering a system of system of polynomial equations and inequalities $S(T,X) = S_{1}(T,X) \triangleright_{1} 0, S_{2}(T,X) \triangleright_{2} 0,..., S_{l}(T,X)\triangleright_{l}0$, where $T=(T_{1},...,T_{p})$ and $X$ are variables with real coefficients. Tarski-Seidenberg states that ``$\exists X$ such that $S(T,X)$" is equivalent to ``$C_{1}(X) \lor...\lor C_{k}(X)$" where $C_{i}(X)$ is a member of a finite list of polynomial equations and inequalities such that, for every $t \in \mathbb{R}^{p}$, $S(t,X)$ has a solution if and only if one $C_{i}(X)$ is satisfied and that the variable $X$ can be eliminated. 
  
\fi
Let $\{ Q_{k}(x) := \sum\limits_{l=1}^{n} q_{kl}x_{l} + q_{k0} : 1 \leq k \leq m\}$,  be a set of affine-linear functions with $q_{ij} \in \mathbb{R}$.
A straight forward application of the Tarski-Seidenberg theorem tells us that the set equality 
\begin{equation}
    \mathbb{R}^{n} = \bigcup\limits_{k = 1}^{m}\{x \in \mathbb{R}^n : Q_{k}(x) \geq 0\},
\end{equation} is equivalent to $(...,q_{kl},...)_{0 \leq l \leq n, 1 \leq k \leq m} \in \mathbb{R}^{m(n+1)}$ satisfying a semialgebraic set. Indeed, the set equality can be interpreted as saying that the first-order formula   
$$\Phi(...,q_{kl},...) := \forall x \in \mathbb{R}^n [(Q_{1}(x) \geq 0 ) \lor \ldots \lor (Q_{k}(x) \geq 0 )] $$
is true. So the set of all $(...,q_{kl},...)_{0 \leq l \leq n, 1 \leq k \leq m}$ such that $\Phi(...,q_{kl},...) = \mathrm{T}$ is a semialgebraic set by the Tarski-Seidenberg theorem. To summarize this discussion, we have 

\begin{proposition}\label{link}

Let  $ Q_{k}(x) := \sum\limits_{l=1}^{n} q_{kl}x_{l} + q_{k0} : 1 \leq k \leq m$, be a collection of affine linear functions whose coefficients are in $\mathbb{R}$. The set $\{(q_{kl}) \in \mathbb{R}^{m(n+1)} : \text{ for every  } x \in \mathbb{R}^n, Q_{1}(x) \geq 0 \text{ or } \ldots \text{ or } Q_{k}(x) \geq 0 \}$ is semialgebraic.
\end{proposition}

We illustrate the above proposition with an example

\begin{example}
Let $Q_{1}(x) = x + y - c$ and $Q_{2}(x) = -x - y + d$.  The regions defined by $Q_{1}(x) \geq 0$ and $Q_{2}(x) \geq 0$ cover $\mathbb{R}^2$ (see figure 1) if and only if $d\geq c$. Therefore the set $\{(c,d) \in \mathbb{R}^2 : \text{ for every } x \in \mathbb{R}^2 \;\; Q_{1}(x) \geq 0 \text{ or } Q_{2}(x) \geq  0\}$ is semialgebraic.
\end{example}

\section{Neural Networks}

In this section we give a formal definition of a feedforward neural network with piece-wise linear activation function. We simply call it a $\emph{neural network}$ throughout this paper. Except for slight differences, most of the material in this section is well-known or easily deduced from existing literature.

\begin{definition}
A \emph{neural network} is a tuple $\mathcal{N} =  (W^{(k)},b^{(k)},t^{(k)})_{1 \leq k \leq L}$, where $W^{(k)} = (w^{(k)}_{ij})$ are $m_{k} \times n_{k}$-matrices and, $b^{(k)} = (b^{k}_{i})$ and $t^{(k)} = (t^{(k)}_{i})$ are vectors of size $m_{k}$, all with real entries. We require that $m_{k} = n_{k+1}$ for $1 \leq k \leq L-1$. We call $\textbf{W} = \{W^{(k)}\}$ the weight matrices, $\textbf{b} = \{b^{(k)}\}$ the bias vectors, and $\textbf{t} = \{t_{k}\}$ the threshold vectors of the neural network. For a given $k$ we call the data $(W^{(k)},b^{(k)},t^{(k)})$ the $k$-th layer of the neural network $\mathcal{N}$.

\end{definition}

Let $\rho_{k} : \mathbb{R}^{n_{k}} \rightarrow \mathbb{R}^{m_k}$ and $\sigma_k : \mathbb{R}^{m_k} \rightarrow \mathbb{R}^{m_k}$ be the affine linear transformations and the threshold functions given by
$$\rho_{k}(x) = W^{(k)}x+b^{(k)} \text{ and } \sigma_{k}(x) = \max\{x,t^{(k)}\}$$
respectively. 
We may associate a function $\nu: \mathbb{R}^{d} \to \mathbb{R}^{p}$, where $d=n_{1}$ and $p=m_L$ to the neural network  (see \cite{zhang2018tropical}) as follows
\begin{equation*}
\nu=\sigma_{L}\circ\rho_{L}\circ\sigma_{L-1}\circ\rho_{L-1}\circ ... \circ\sigma_{1}\circ\rho_{1}
\end{equation*}

We can consider a simple neural network, such as an Autoencoder \cite{schmidhuber2015deep}, depicted in Figure 1, to clarify our neural network definition. This Autoencoder corresponds to the function $\nu=\sigma_{2}\circ\rho_{2}\circ\sigma_{1}\circ\rho_{1}$ as follows
\begin{align*}
    \sigma_{2} &= \max\{x,t^{(2)}\}, t^{2} = \begin{bmatrix} 1 & 1 & 1 \end{bmatrix}\\
    \rho_{2} &= W^{(2)}x+b^{(2)},  W^{(2)} = \begin{bmatrix} 3 & 4 & -5 \\ \frac{1}{2} & -\frac{3}{4} & 6\end{bmatrix}, b^{2} = \begin{bmatrix} \frac{1}{3} & \frac{1}{3} & \frac{1}{3} \end{bmatrix}\\
    \sigma_{1} &= \max\{x,t^{(1)}\}, t^{1} = \begin{bmatrix} 0 & 0 & 0 \end{bmatrix}\\
    \rho_{1} &= W^{(1)}x+b^{(1)}, W^{(1)} = \begin{bmatrix} 1 & 3 & -7 \\ -\frac{1}{3} & \frac{3}{8} & -6\end{bmatrix}, b^{1} = \begin{bmatrix} 2 & 2 & 2 \end{bmatrix}
\end{align*}

Note that this neural network has two layers by our definition but three layers in the classical sense, as depicted in the figure.

\begin{figure}
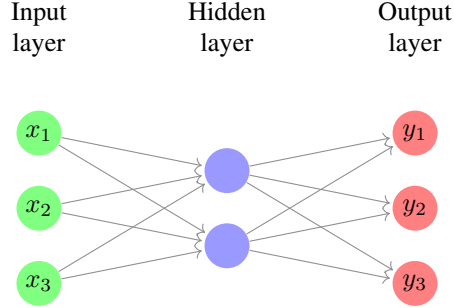

\centering
\begin{neuralnetwork}[height=4]
		\newcommand{\nodetextclear}[2]{}
		\newcommand{\nodetextx}[2]{$x_#2$}
		\newcommand{\nodetexty}[2]{$y_#2$}
		\inputlayer[count=3, bias=false, title=Input\\layer, text=\nodetextx]
		\hiddenlayer[count=2, bias=false, title=Hidden\\layer, text=\nodetextclear] \linklayers
		\outputlayer[count=3, title=Output\\layer, text=\nodetexty] \linklayers
	\end{neuralnetwork}
\caption{Example of an Autoencoder.}
\label{fig_m_3}
\end{figure}

\begin{remark} A neural network may have different weight matrices, biases and thresholds but may have identical associated functions. In such a case we want to identify different neural networks with the same associated functions. 
\end{remark}

\begin{definition}[Neural Network Equivalence]
We say that two neural networks $\mathcal{N}_{1}$ and $\mathcal{N}_{2}$ are equivalent if their associated functions $\nu_{1}$ and $\nu_{2}$, respectively, are identically equal.
\end{definition}

In the rest of this section we demonstrate how we can write a function associated to a neural network as difference between two piece-wise linear functions in a recursive way.

\begin{remark} \label{PL}
\begin{itemize}
    \item[(1)]   Let $\mathrm{W}$ be a matrix with \emph{non-negative} entries and let $$ T(x) = \max \{ T_{1}(x), \ldots T_{m}(x) \}$$ be  a piece-wise linear function, i.e., $T_{i} : \mathbb{R}^n \rightarrow \mathbb{R}^m$'s are coordinate wise affine-linear functions. For any vector $b \in \mathbb{R}^{m}$, the function $g(x) = \max \{\mathrm{W}.T(x)+b,t\}$ is a piece-wise linear function. Indeed, we can show that $\mathrm{W}.T(x) = \\ \max \{T_{ij}(x)\}$, where $T_{ij}(x)$'s are affine-linear functions described as follows. 
         
        For simplicity we assume $n=m=2$. We calculate $$\mathrm{W} \max\limits_{i} \{a^{(i)}(x)\},$$ where $a^{(i)}(x) = \begin{pmatrix} a^{(i)}_{1}(x)\\
                               a^{(i)}_{2}(x)
                               \end{pmatrix}$ are coordinate-wise affine-linear functions (may be constant) and the max is taken coordinate-wise. We have
         \begin{align*}
             &\mathrm{W} \max\limits_{i} \{a^{(i)}(x)\}  = \begin{pmatrix}
                                     w_{11} & w_{12} \\
                                     w_{21} & w_{22} 
                               \end{pmatrix} . \max\limits_{i}\{\begin{pmatrix} a^{(i)}_{1}(x)\\
                               a^{(i)}_{2}(x)
                               \end{pmatrix}\} \\
                               & =  \max\limits_{i,j} \{\begin{pmatrix} w_{11}a^{(i)}_{1}(x) + w_{12}a^{(j)}_{2}(x)\\
                               w_{21}a^{(i)}_{1}(x) + w_{22}a^{(j)}_{2}(x)
                               \end{pmatrix}\}
         \end{align*}
          Here the matrix product commutes with the max operation since we assumed that the $w_{ij}$'s are non-negative.
    \item[(2)]   When the $w_{ij}$'s are arbitrary real numbers, there exists a decomposition  $w_{ij} = (w_{ij})_{+} - (w_{ij})_{-}$, where $(w_{ij})_{+} = \max \{ w_{ij}, 0 \}$ and $(w_{ij})_{-} = \max \{-w_{ij}, 0\}$. This lets us write $\mathrm{W} = \mathrm{W}_{+} - \mathrm{W}_{-}$, where  $\mathrm{W}_{+} = ((w_{ij})_{+})$ and $\mathrm{W}_{-} = ((w_{ij})_{-})$. So we have
      
             $$ \mathrm{W}.T(x) = \mathrm{W}_{+}T(x) - \mathrm{W}_{-}T(x) $$\\
                               
          a difference between two piece-wise linear functions.      
          
\end{itemize}

           \end{remark}
    \iffalse
    
    Let $\nu_{L} = \sigma_{L} \circ \rho_{L} \circ \ldots \circ \sigma_{1} \rho_{1}$ be the function associated to the neural network $\mathcal{N}$. We may rewrite it as 
    \begin{equation*}
    \begin{split}
           \nu_{L}(x) & = \max\{ \mathrm{W}^{(L)}\nu_{L-1}(x) + b^{(L)}, t^{(L)} \} \\
        & = \max\{ \mathrm{W}_{+}^{(L)}\nu_{L-1}(x) - \mathrm{W}_{-}^{(L)}\nu_{L-1}(x) + b^{(L)}, t^{(L)} \} \\
        & = \underbrace{\max \{\mathrm{W}_{+}^{(L)}\nu_{L-1}(x) + b^{(L)} , t^{L} + \mathrm{W}_{-}^{(L)}\nu_{L-1}(x) \}}_{(\nu_{L})_{+}(x)} - \underbrace{\mathrm{W}_{-}^{(L)}\nu_{L-1}(x)}_{(\nu_{L})_{-}(x)}. \\
    \end{split}
    \end{equation*}
    
  So that we may write $$(\nu_{L})(x) = (\nu_{L})_{+}(x) - (\nu_{L})_{-}(x).$$ We call the functions $(\nu_{L})_{+}$ and $(\nu_{L})_{-}$ the \emph{positive part} and \emph{negative part} of $\nu_{L}$ respectively.  Note that both $(\nu_{L})_{+}(x)$ and $(\nu_{L})_{-}(x)$ are piece-wise linear functions by Remark \ref{PL}. 
    \fi

\begin{lemma}\label{split}
The function $\nu_{L}(x)$ associated to a  neural network $\mathcal{N}$ can be written as $$(\nu_{L})(x) = (\nu_{L})_{+}(x) - (\nu_{L})_{-}(x),$$
where $(\nu_{L})_{+}$ and $(\nu_{L})_{-},$ are piece-wise linear functions called the \emph{positive part} and \emph{negative part} of $\nu_{L}$ respectively.
\end{lemma}

\begin{proof}
We proceed by induction on $L$, the number of layers. The base case is trivial. For the inductive step we may assume $$\nu_{L-1} = (\nu_{L-1})_{+} - (\nu_{L-1})_{-},$$ 
where $(\nu_{L-1})_{+}$ and $(\nu_{L-1})_{-}$ are piece-wise linear. On the other hand, we have $\mathrm{W}^{(L)} = \mathrm{W}^{(L)}_{+} - \mathrm{W}^{(L)}_{-}$, where the entries of $\mathrm{W}^{(L)}_{+}$ and $\mathrm{W}^{(L)}_{-}$ are non-negative. In the following simplification we drop the indices to make it more readable
\iffalse
\begin{align*}
    &\nu_{L}(x) \\
    & = \max\{ \mathrm{W}^{(L)}\nu_{L-1}(x) + b^{(L)},t^{(L)} \} \\
    & = \max\{( \mathrm{W}_{+} - \mathrm{W}_{-}) (\nu_{+}(x) - \nu_{-}(x)) + b,t \} \\
    & = \underbrace{\max \{ \mathrm{W}_{+}\nu_{+}(x) + \mathrm{W}_{-}\nu_{-}(x) + b , t + \mathrm{W}_{-}\nu_{+}(x) + \mathrm{W}_{+}\nu_{-}(x) \}}_{(\nu_{L})_{+}(x)} \\
    &\phantom{={}} -  \underbrace{(\mathrm{W}_{-}\nu_{+}(x)   +
    \mathrm{W}_{+}\nu_{-}(x))}_{(\nu_{L})_{-}(x)}.
\end{align*}
\fi
\begin{align*}
    \nu_{L}(x)= &\max\{ \mathrm{W}^{(L)}\nu_{L-1}(x) + b^{(L)},t^{(L)} \} \\
    = &\max\{( \mathrm{W}_{+} - \mathrm{W}_{-}) (\nu_{+}(x) - \nu_{-}(x)) + b,t \} \\
     = &\max \{ \mathrm{W}_{+}\nu_{+}(x) + \mathrm{W}_{-}\nu_{-}(x) + b , \\
     &t + \mathrm{W}_{-}\nu_{+}(x) + \mathrm{W}_{+}\nu_{-}(x) \} \\
    &- (\mathrm{W}_{-}\nu_{+}(x)   +
    \mathrm{W}_{+}\nu_{-}(x))
\end{align*}

By Remark \ref{PL}, both $(\nu_{L})_{+}(x)$ and $(\nu_{L})_{-}(x)$ are piece-wise linear. Hence the lemma.

\end{proof}

To summarize, $\nu_{L}(x) =  (\nu_{L})_{+}(x) - (\nu_{L})_{-}(x)$, where 

\begin{align*}
(\nu_{L})_{+}(x) &= \max \{ \mathrm{W}^{(L)}_{+}(\nu_{L-1})_{+}(x) +\\  &\mathrm{W}^{(L)}_{-}(\nu_{L-1})_{-}(x) + b^{(N)} , t^{(N)} + \\
&\mathrm{W}^{(L)}_{-}(\nu_{L-1})_{+}(x) + \mathrm{W}^{(L)}_{+}(\nu_{L-1})_{-}(x) \}\\
(\nu_{L})_{-}(x) &= (\mathrm{W}^{(L)}_{-}(\nu_{L-1})_{+}(x)   + \mathrm{W}^{(L)}_{+}(\nu_{L-1})_{-}(x))
\end{align*}

The base case (single layer) is given by $$\nu_{1}(x) = \sigma_{1}\circ \rho_{1}(x) = \max \{\mathrm{W}^{(1)}x + b^{(1)}, t^{(1)}\}$$

\begin{remark}\label{key}
\begin{itemize}
    \item[(1)]  Note that the recursive formulae for $\nu_{L}$, along with Remark \ref{PL} allows us to compute $\nu_L$ explicitly in terms of the coefficients of the neural network. More explicitly, we may write 
    %\begin{equation*}\label{rep}
        \begin{align*} 
          (\nu_{L})_{+}(x) &= \max\limits_{k \in S} \{ A_{k}(x) \}\\
          &\text{    and    }\\
          (\nu_{L})_{-}(x) &= \max\limits_{k \in T}\{B_{k}(x)\},
        \end{align*}
    %\end{equation*}
   where $A_k(x) = \sum\limits_{i=1}^{d}a_{ik}x_{i}+a_{0k}$ and $B_k(x) = \sum\limits_{i=1}^{d}b_{ik}x_{i}+b_{0k}$ are affine linear functions whose coefficients are \emph{polynomial expressions}  in the entries of  the weight matrices $\{\mathrm{W}_{+}^{(i)}, \mathrm{W}_{-}^{(i)}\}_{1 \leq i \leq L}$. A formal proof of this is a straightforward application of induction to the recursive formulae above. The recursive formulae also allows us to write these coefficients explicitly.
    \item[(2)] We may treat the coefficients $a_{ik}$ and $b_{ik}$ as formal polynomials whose variables are the entries of the weight matrices $\{\mathrm{W}_{+}^{(i)}, \mathrm{W}_{-}^{(i)}\}_{1 \leq i \leq L}$. This immediately implies that the indices $S$ and $T$ in Remark \ref{key} (1) are dependent only on the network architecture, i.e., the number of layers and the number of nodes in each layer of the neural network. 
\end{itemize}
\end{remark}

\section{Characterizing Equivalent Neural Networks}

In this section we introduce a few notions related to piece-wise linear functions and prove our main theorem. Given a neural architecture, we reduce the problem of characterizing the neural networks $\mathcal{N}$ equivalent to a given neural network $\mathcal{N}_0$ to that of a cancellation problem for piece-wise linear functions. More precisely, it amounts to finding conditions for coefficients of affine linear functions $\{A_{k}(x)\}_{k}$ and $P(x)$ such that 
$$f(x) := \max\limits_{k}\{A_{k}(x), P(x) \}  \equiv \max\limits_{k}\{A_{k}(x)\}$$
We call affine linear terms such as $P(x)$ \emph{redundant} for the function $f(x)$. In theorem \ref{tsapp} we give a criterion for $P(x)$ to be redundant. We further observe that the set of \emph{relevant} terms, i.e., the terms which are not redundant, are unique for a given piece-wise linear function, see theorem \ref{uniquemin}. Using these results we prove that the set of neural networks equivalent to a given neural network $\mathcal{N}_{0}$ is given by a semialgebraic set, see theorem \ref{main}.

\iffalse
--------------------------------------------------------
\begin{definition}[Relevant Index]
Let $f(x) = \max\limits_{k\in S}\{P_{k}(x)\}$ be a piece-wise linear (PL) function. We say that an index $l \in S$ is \emph{relevant} if $f(x_0) \neq \max\limits_{k \in S \setminus \{l\}}\{P_{k}(x_{0})\}$ for some $x_{0} \in \mathbb{R}^{n}$.
We say an affine linear function $Q(x)$ is \emph{redundant} with respect to a piecewise linear function $f(x)$, if for all $x \in \mathbb{R}^n$, $f(x) = \max\{f(x),Q(x)\}$. We say that $f(x) = \max_{k \in S}\{P_{k}(x)\}$ is a minimal representation of $f$ if every index in $S$ is relevant and $P_{i} \neq P_{j}$ for $i \neq j$.  
\end{definition}

\begin{center}
    (vs)
\end{center}
-----------------------------------------------------------

\begin{notation}\label{link-notation}
Let $P_{k}(x) = \sum\limits_{l=1}^{d} a_{kl}x_{l} + a_{k0}$, for $k \in I$, be a finite collection of distinct affine-linear functions with coefficients in $\mathbb{R}$. Fix an index $j \in I$ and let $T_{jk} := \{x \in \mathbb{R}^d : P_{j}(x) \leq P_{k}(x)\}.$ We denote by $\mathcal{S}_j$ the set of coefficients $(a_{kl}) \in \mathbb{R}^{(d+1) \times |I|}$ such that $j \in I$ is a redundant index for $f(x) = \max\limits_{k \in I} \{P_{k}(x)\}$. 
\end{notation}
\fi

\begin{definition}[Redundant index]\label{redindex} 
Let $S$ be a finite set, $P_{k} : \mathbb{R}^d \rightarrow \mathbb{R}$, for $k \in S$ be distinct affine-linear functions given by $P_{k}(x) = \sum\limits_{l=1}^{d} a_{kl}x_{l} + a_{k0}$ and  $f(x) = \max\limits_{k\in S}\{P_{k}(x)\}$, a piece-wise linear function. We say that an index $j \in S$ is \emph{redundant} if the following identical equality of functions is valid $$ f(x) = \max\limits_{k\in S}\{P_{k}(x)\} \equiv   \max\limits_{k\in S, k \neq j} \{P_{k}(x)\}.$$ The indices in $S$ which are not redundant are called \emph{relevant} indices of $f$. We say that a set $J \subset S$ is a \emph{full set of relevant indices} if all the indices in $J$ are relevant and all those in $J^{c}$ are redundant. In the case where $S$ is a full set of relevant indices we say that $f(x)  = \max\limits_{k\in S}\{P_{k}(x)\}$ is a minimal representation of $f$.
\end{definition}

\begin{remark}\label{redundant}
  With the notation in definition \ref{redindex} observe that $j \in S$ is a redundant index for  $f(x) = \max\limits_{k\in S}\{P_{k}(x)\}$ if and only if for every $x \in \mathbb{R}^{d}$ there exists an index $k \in S \setminus \{j\}$ such that $P_{k}(x) \geq P_{j}(x)$. Fix an index $j \in S$ and let $T_{jk} := \{x \in \mathbb{R}^d : P_{j}(x) \leq P_{k}(x)\}.$ We denote by $\mathcal{S}_j$ the set of coefficients $(a_{kl}) \in \mathbb{R}^{(d+1) \times |S|}$ such that $j \in S$ is a redundant index for $f(x) = \max\limits_{k \in S} \{P_{k}(x)\}$. In other words $j \in S$ being redundant is equivalent to the following set equality $$ \mathbb{R}^{d} = \bigcup\limits_{k \in S, k \neq j} T_{jk}.$$ 
\end{remark}

\begin{example}
Consider the piece-wise linear function $f(x) = \max\{ x, 2x, ax \}$. The term $ax$ in $f(x)$ is redundant if and only if $1 < a < 2$. Similarly, we can find conditions for parameters $a_{3}$ and $b_{3}$ such the term $a_{3}x + b_{3}$ in $f(x) = \max\{a_{1}x+b_{1}, a_{2}x+b_{2}, a_{3}x+b_{3}\}$ is redundant. These conditions are necessarily semialgebraic relations by the following theorem. For an explanation of this fact in general, see figure \ref{example}. 
\end{example}

\begin{figure}
\centering
\includegraphics[scale=0.25]{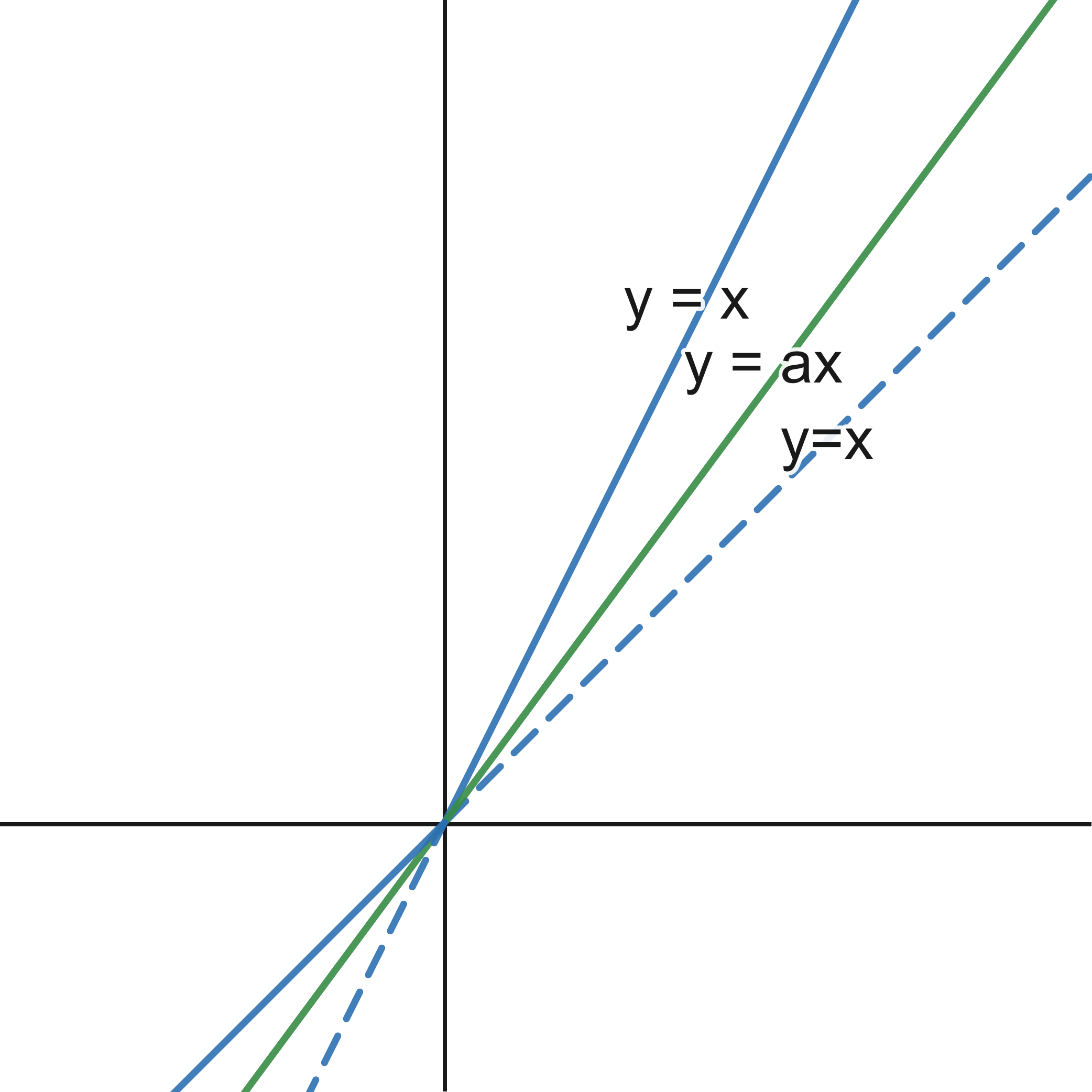}
\caption{The line $y=ax$ in the middle is redundant}
\label{fig_m_4}
\end{figure}

\begin{theorem}\label{tsapp}
Let $h(x) = \max\limits_{k \in S}\{P_{k}(x)\}$ be a piece-wise linear function, i.e. the pieces $P_{k}(x)$ are affine linear for $k \in S$. An affine-linear piece $P_{j}(x)$ of $h(x)$ is redundant if and only if the space of coefficients $\mathcal{S}_{j} $ is a semialgebraic set.
\end{theorem}
\begin{proof}
A term $P_{j}(x)$ is redundant if and only if for every $x \in \mathbb{R}^d$ there exists an index $k \neq j$ such that $P_{j}(x) \leq P_{k}(x)$. Therefore, $P_{j}(x)$ is redundant if and only if $\mathbb{R}^{d} = \cup_{k \in I, k \neq j} T_{jk} $. In other words,  $$\mathcal{S}_{j} = \{ (a_{kl}) : \mathbb{R}^d =  \bigcup\limits_{k \in I \setminus \{j\}} T_{jk}\}.$$ Therefore, by proposition \ref{link} $\mathcal{S}_{j}$ is a semialgebraic set. 

% But  $\mathbb{R}^{d} = \cup\limits_{k \in I, k \neq j}T_{jk} $ is equivalent to $\mathcal{S}_j$ being semialgebraic by proposition \ref{link}.
\end{proof}

\begin{figure}
\centering
\includegraphics[scale=0.22]{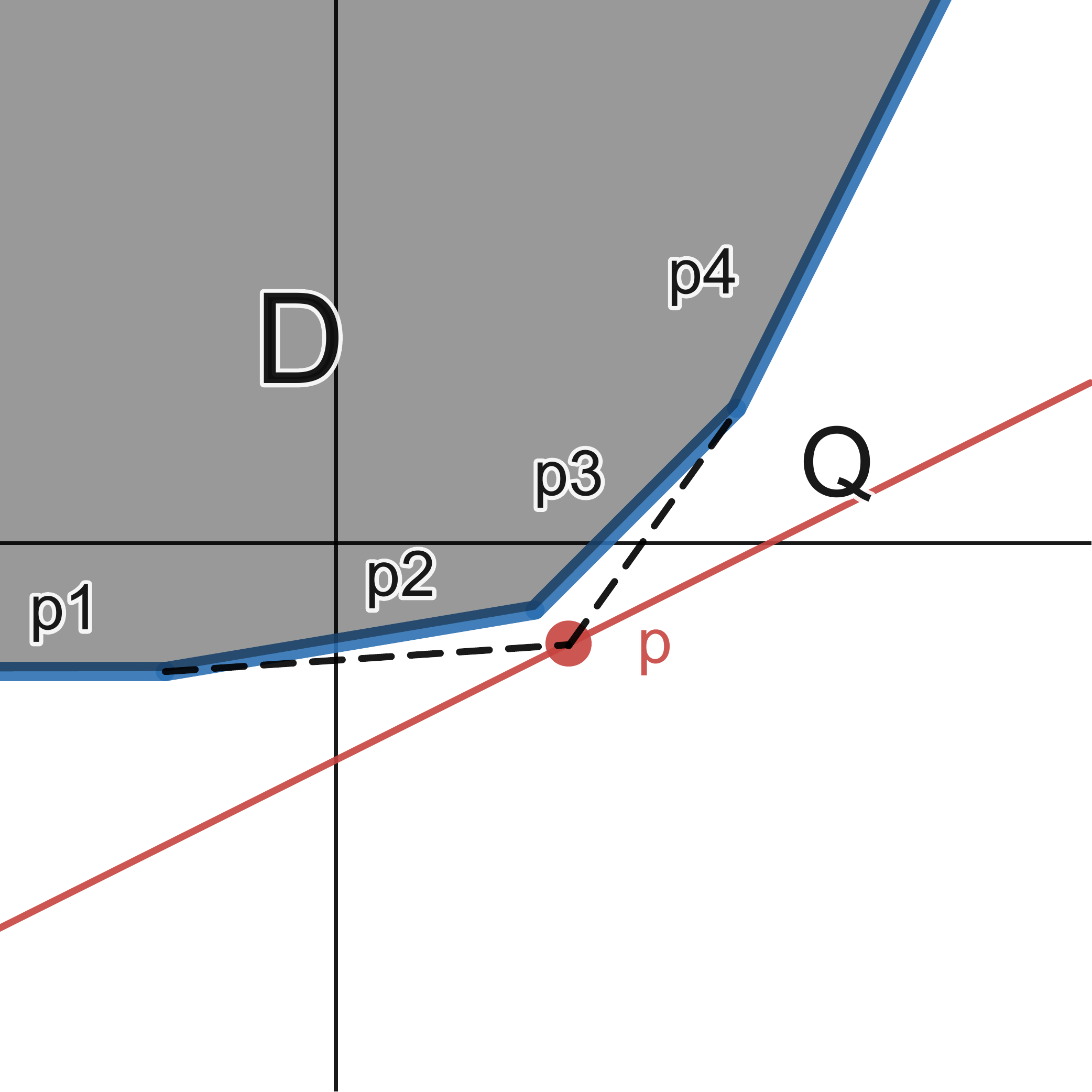}
\caption{The figure represents a piece-wise linear function $f(x) = \max\{p_{1},p_{2},p_{3},p_{4}\}$ in $\mathbb{R}^{2}$. All the pieces are relevant terms in this piece-wise linear representation for $f$. $Q$ is a redundant term for $f(x)$ iff the line represented by $Q$ does not pass through the convex region $D$.\\\\
Fix a point $p$ in the complement of $D$. A line $Q$ passing through $p$ is redundant for $f(x)$ iff $Q$ does not cut the convex hull $<D,P>$ spanned by $D$ and $p$. This easily shows that the slope $m$ of $Q$ for which $Q$ is redundant satisfies a semialgebraic set. Moreover, as we vary $p$ in the complement of $D$ we get another semialgebraic relation involving the intercepts of $Q$. Hence $Q$ is redundant iff the coefficients of $Q$ satisfy a semi-algebraic relation.  } 
\label{example}
\end{figure}

\begin{remark} \label{Wdep}
\begin{itemize}
    \item[(1)] A linear term in a PL function is either redundant or relevant. Since semialgebraic sets form a Boolean algebra, the complement of $\mathcal{S}_{j}$ gives conditions for a term $P_{j}$ to be relevant as well. More explicitly, $j \in S$ is relevant if and only if $\mathcal{N} \in \mathcal{S}^{c}_{j}$.
    \item[(2)]As the coefficients of the linear terms in theorem \ref{tsapp} varies, a given linear term $P_{k}(x)$ may change its status from being relevant to redundant and vice-versa. As a result the number of relevant terms in $h(x)$ may vary with its coefficients.  
\end{itemize}
\end{remark}

\begin{lemma}\label{help}
With the notation in theorem \ref{tsapp}, a subset $K \subset S$ is a full set of relevant indices for $h(x)$ if and only if  $\mathcal{N} \in \mathcal{S}_{K} := (\cap_{k \in K^{c}} \mathcal{S}_{k}) \cap (\cap_{k \in K} \mathcal{S}^{c}_{k})$.
\end{lemma}
\begin{proof}

By remark \ref{Wdep} we know that $k \in I$ is a relevant index if and only if $\mathcal{N} \in \mathcal{S}^{c}_{k}$. Therefore, every $k \in K$ is a relevant index if and only if $\mathcal{N} \in \cap_{k \in K} \mathcal{S}^{c}_{k}$. Similarly, by theorem \ref{tsapp} every $k \in K^{c}$ is redundant if and only if   $\mathcal{N} \in \cap_{k \in K^{c}} \mathcal{S}_{k}$.  

 By definition $K$ is a full set of relevant indices if and only if every index in $K$ is relevant and every index in $K^{c}$ is redundant. Therefore, $K$ is a full set of relevant indices if and only if $\mathcal{N} \in ( \cap_{k \in K} \mathcal{S}^{c}_{k}) \cap (\cap_{k \in K^{c}} \mathcal{S}_{k})  $.
\end{proof}

\begin{definition}[Corner]
A corner of a piecewise linear function $f(x) = \max\limits_{k\in S}\{P_{k}(x)\}$ is some $x_{0}\in\mathbb{R}^{d}$ such that $f(x_0) = P_{i}(x_0) = P_{j}(x_0)$ for some $i\neq j$. 
%We want to more than this, not all corners qualify as 'true corners' even in single variable case.   
\end{definition}

\begin{figure}
\centering
\includegraphics[scale=0.3]{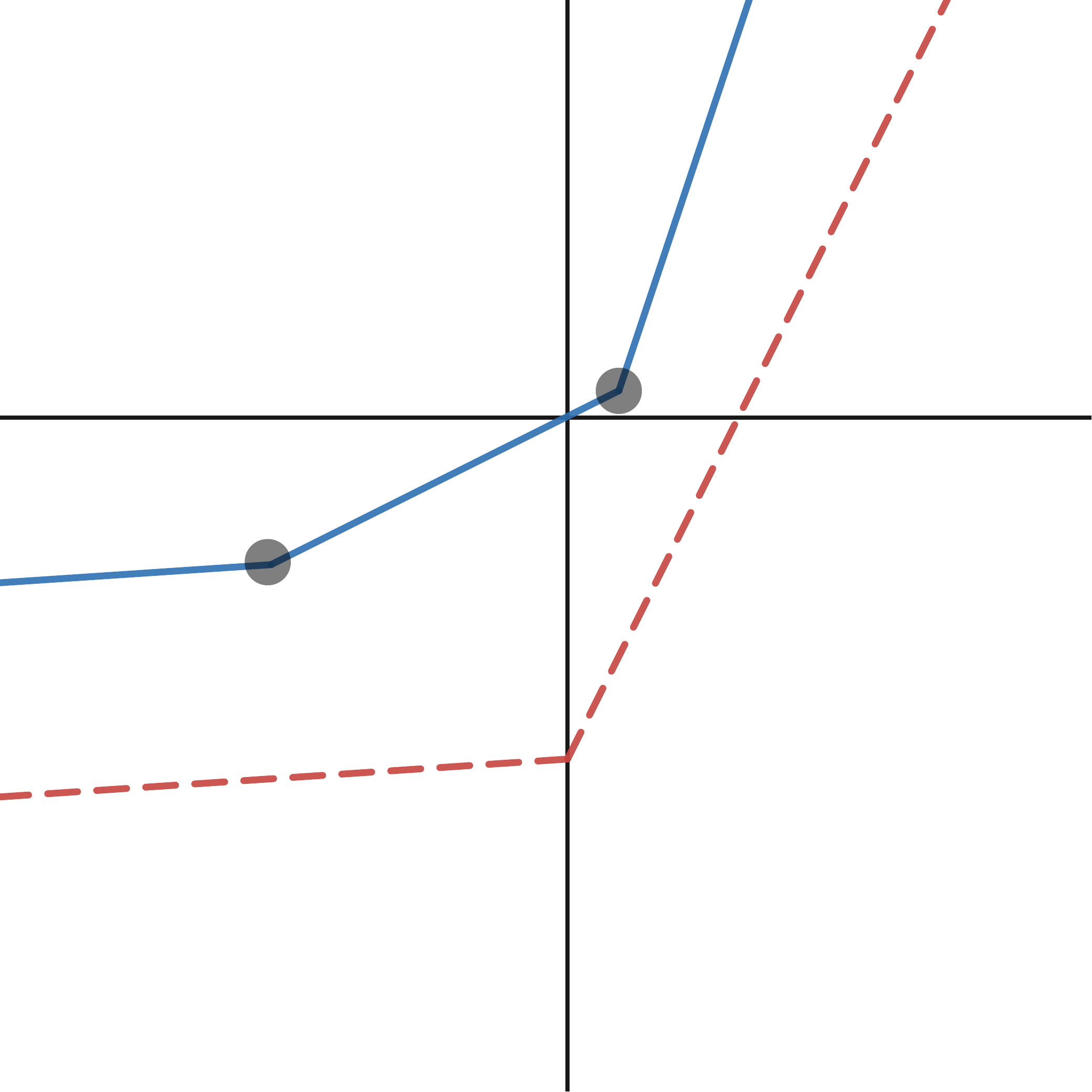}
\caption{This figure represents a piece-wise linear function. The upper pieces represent the relevant terms and the lower dotted lines represent two redundant terms. The circled points are the corners. Note that the point of intersection of the redundant pieces is not a corner.}
\label{fig_m_4}
\end{figure}

\iffalse
We can compose an algorithm for finding all corners in the single-variable case. Assume $x_{1}<x_{2}<...<x_{m}$ are the x-coordinates of the corners. We observe that $x_{1} = min\{\frac{a_{0}-a_{k}}{k}\mid 1\leq k\leq N\}$. We can extend this to see that $x_{2} = min\{\frac{a_{k}-a_{l}}{l-k}\mid l\in \{1,...,N\}\setminus\{k_{1}\} \}$. We can iteratively derive $x_{k}$ in this way.
\fi

\iffalse    
\begin{lemma}
If $ \max\{P,P_{2},\ldots P_{n}\} \equiv \max\{P,Q_{2},\ldots Q_{n}\} $ such that all the affine-linear terms appearing on both sides are relevant  then $ \max\{P_{2},\ldots P_{n}\} \equiv \max\{Q_{2},\ldots Q_{n}\} $.
\end{lemma}
\begin{proof}
Suppose $P_{i}$
\end{proof}
\fi

\begin{theorem}\label{uniquemin}
Every piecewise linear function has a unique minimal representation. 
\end{theorem}

\begin{proof}
Let $f(x) = \max\limits_{k\in S}\{P_{k}(x)\} = \max\limits_{k\in T}\{Q_{k}(x)\}$ be minimal representations of $f$. In particular, this means that $S$ and $T$ are full set of relevant indices of $f$. Let $D_{k} :=\{x\in\mathbb{R}^{n}\mid f(x) = P_{k}(x)\}$ and $D^{\prime}_{k} := D_{k}\setminus\{\text{corners}\}$. Since every index is relevant, $D_{k}' \subset \mathbb{R}^d$ is a non-empty open subset. Similarly $f(z) = Q_{l}(z)$ in some open subset $E^{\prime}_{l} \subset \mathbb{R}^d$ containing $x$.   

Therefore, for all $x \in D^{\prime}_{k}$ there exists a neighbourhood $B_{x}$ such that $P_{k}(z) = Q_{l}(z) \text{ for some } k \in S, l\in T, \forall z\in B_{x}$.  Since $B_{x}$ is an open neighbourhood and $P_{k}$ and $Q_{l}$ are affine-linear functions in $\mathbb{R}^{d}$, $P_{k}(z) = Q_{l}(z), \forall z \in \mathbb{R}^{d}$.
\end{proof}

\iffalse
\begin{proof}
Let \begin{equation}
    f(x) = \max\limits_{k\in S}\{P_{k}(x)\} \equiv \max\limits_{k\in T}\{Q_{k}(x)\}
\end{equation} be minimal representations of $f$. Let $i \in S$ and $j \in T$, define $$C_{ij} := \{ x \in \mathbb{R}^d :  P_{i}(x) = Q_{j}(x) \}. $$
Observe that for every $x \in \mathbb{R}^{d}$, there exists $i \in S$ and $j \in T$ such that $x \in C_{ij}$. Therefore, $\bigcup\limits_{i,j} C_{ij} =  \{x \in \mathbb{R}^d : \max\limits_{k\in S}\{P_{k}(x)\} = \max\limits_{k\in T}\{Q_{k}(x)\} \} = \mathbb{R}^d$.
\end{proof}
\fi

\begin{example}
For a single layer neural network each equivalence class has a unique neural network. This is a consequence of Theorem \ref{uniquemin}. The piecewise linear interpretation of this is the observation that if $\max\{ax+b,t\} = \max\{cx+d,s\}$ for all $x \in \mathbb{R}^{n}$ then $a = c, b=d$, and $t=s$.
\end{example}

For the rest of this section, fix a neural network $\mathcal{N}_{0} = (W_{0}^{(k)},b_{0}^{(k)},t_{0}^{(k)})_{1 \leq k \leq L}$, where $W_{0}^{(k)} = ((w_{0})^{(k)}_{ij})$, with associated function $f_{0}$. We also assume that all the neural networks appearing in this section have the same neural architecture, i.e., the number of layers $L$ and the number of nodes in each layer remain the same.  We denote the space of all neural networks $\mathcal{N} = (W^{(k)},b^{(k)},t^{(k)})_{1 \leq k \leq L}$ equivalent to $\mathcal{N}_{0}$ by $\mathrm{NN}_{f_{0}}$. Suppose $f$ is the function associated to a $\mathcal{N} \in \mathrm{NN}_{f_{0}}$, then by definition $f \equiv f_{0}$. Using lemma \ref{split} we have $f_{+} - f_{-} \equiv (f_{0})_{+} - (f_{0})_{-}$ and hence
\begin{align}\label{plrelation}
    f_{+} + {(f_{0})}_{-} \equiv f_{-} + (f_{0})_{+}
\end{align}
where $f_{+}(x) = \max\limits_{k}\{P_{k}(x)\}$, $f_{-}(x) = \max\limits_{k}\{Q_{k}(x)\}$ and $(f_{0})_{+}(x) = \max\limits_{k}\{N_{k}(x)\}$, $(f_{0})_{-}(x) = \max\limits_{k}\{M_{k}(x)\}$ are piece-wise linear functions. By equations \ref{rec} in section 3 the affine linear functions $P_{k}, Q_{k}, N_{k}$ and $M_{k}$ have coefficients in terms of the data $\mathcal{N} = (W^{(k)},b^{(k)},t^{(k)})_{1 \leq k \leq L}$ and $\mathcal{N}_{0} = (W_{0}^{(k)},b_{0}^{(k)},t_{0}^{(k)})_{1 \leq k \leq L}$. It is easy to see that equation \ref{plrelation} can be rewritten as  
\begin{align} \label{PLeq}
    \max\limits_{(k,l) \in I}\{A_{k,l}(x)\} = \max\limits_{(k,l) \in J}\{B_{k,l}(x)\}
\end{align}
where $A_{k,l}(x) = P_{k}(x) + N_{l}(x) $ and $B_{k,l}(x) = Q_{k}(x) + M_{l}(x)$.

\begin{remark}
The indices $I$ and $J$ are determined by the neural architecture. In other words, for a given neural architecture the number of linear pieces in both sides of equation \ref{PLeq} is independent of $\mathcal{N}$.
\end{remark}

From the above discussion a neural network $\mathcal{N}$ is equivalent to $\mathcal{N}_{0}$ if and only if equation \ref{PLeq} is satisfied. In other words, finding the space of neural networks $\mathcal{N}$ equivalent to a fixed neural network $\mathcal{N}_{0}$ is the same as finding conditions on the coefficients of $A_{k,l}(x)$ and $B_{k,l}(x)$ such that equation \ref{PLeq} is satisfied. In what follows we prove that the coefficients of $A_{k,l}(x)$ and $B_{k,l}(x)$ are given by a semialgebraic set using the Tarski-Seidenberg theorem.

By theorem \ref{uniquemin}, for any given $\mathcal{N}$ the minimal representations of both sides of equation \ref{PLeq} coincide. Therefore, there could be redundant terms among $\{A_{kl}(x)\}$ and $\{B_{kl}(x)\}$. The possibility of such redundancy is the crux of the argument behind the main theorem.

\begin{remark}\label{bij}
  Considering Eq.(\ref{PLeq}), for a given $\mathcal{N}$, theorem \ref{uniquemin} implies that there exists subsets $K \subset I$ , $K' \subset J$, and a bijection $\sigma : K \rightarrow K'$ such that $A_{k}(x) \equiv B_{\sigma(k)}(x)$ for every $k \in K$. These $K$, $K'$ and $\sigma$ vary as $\mathcal{N}$ varies, see remark \ref{Wdep}.
\end{remark}

\begin{lemma}\label{bridge}
Let $P(x_{1}, \ldots, x_{n})$ be a polynomial over $\mathbb{R}$. The set $$S_{P} := \{(a_1, \ldots a_{n}) \in \mathbb{R}^n : P(\max\{a_{1},0\} \ldots \max\{a_{n},0\}) \geq 0\}$$ is semialgebraic. 
\end{lemma}
\begin{proof}
\iffalse
We proceed by induction. For $n=1$, in the region $R_{1} = \{a : a \geq 0 \})$, we have $P(\max\{a,0\}) = P(a)$. hence $S_{P} \cap R_{1}$ is semialgebraic. Similarly, in $R_{2} =  \{a : a < 0 \}) $, we have $P(\max\{a,0\}) = P(0)$ making $S_{P} \cap R_{2} = \emptyset$ or $R_{2}$ which is semialgebraic as well. Therefore $S_{P} = (S_{P} \cap R_{1}) \cup (S_{P} \cap R_{2})$ is semialgebraic. 
Suppose the statement is true for any $n$ variable polynomial. Let $P(x_{1}, \ldots x_{n+1})$ be a polynomial in $n+1$ variables. 
\fi
Let $m_1, \ldots, m_n \in \{0 ,1\}$ and $R(m_1, \ldots, m_n) := \{(a_{1}, \ldots, a_{n}) \in \mathbb{R}^n : a_{i} > 0 \text{ if } m_{i} = 1 \text{ and  }  a_{i} \leq 0 \text{ if } m_{i} = 0\} .$
Notice that $R(m_1, \ldots, m_n) \subset \mathbb{R}^n$ are the \emph{orthants} of $\mathbb{R}^n$. Observe that $S_P \cap R(m_{1}, \ldots m_{n})$ is semialgebraic for any choice of $m_{i} \in \{0,1\}$. Therefore, $S_P = \bigcup_{(m_1, \ldots, m_n) \in \{0,1\}^{n}} (S_P \cap R(m_1, \ldots, m_n))$ is semialgebraic.
\end{proof}

\begin{theorem}\label{main}
The space of coefficients of neural networks equivalent to a given neural network $\mathcal{N}_{0}$ is given by a semialgebraic set denoted by $\mathcal{S}_{\mathcal{N}_{0}}$.  
\end{theorem}    
\begin{proof}
If $\mathcal{N}$ and $\mathcal{N}_0$ are equivalent we argued that Eq.(\ref{PLeq}) is true. By lemma \ref{help}, $K \subset I$ is a full set of relevant indices of the LHS of Eq.(\ref{PLeq}) if and only if $\mathcal{N} \in \mathcal{S}_{K}$. Similarly, $K' \subset J$ is a full set of relevant indices of the RHS if and only if $\mathcal{N} \in \mathcal{S}_{K'}$. Further, by remark \ref{bij}, if $A_{k}(x) = \sum_{l=1}^{d}a_{kl}x_{l} + a_{k0}$ and $B_{k}(x) = \sum_{l=1}^{d}b_{kl}x_{l} + b_{k0}$ then $a_{kl} = b_{\sigma(k)l}$ for every $k \in K$. This implies that $\mathcal{N} \in V(K,K',\sigma) := \{ \mathcal{N} :  a_{kl} = b_{\sigma(k)l}, \text{ for every } k \in K \}$. Therefore, $\mathcal{N} \in \mathcal{S}_{K,K',\sigma} : = \mathcal{S}_{K} \cap \mathcal{S}_{K'} \cap V(K,K',\sigma)$

Conversely, suppose $\mathcal{N}$ is such that there exists $K \subset I$, $K' \subset J$, and a bijection $\sigma : K \rightarrow K'$ so that $K$ and $K'$ are the full set of relevant indices for the LHS and RHS of Eq.(\ref{PLeq}) respectively. Also suppose that the coefficients of the linear terms corresponding to $K$ and $K'$ coincide with respect to the ordering given by $\sigma$, i.e., $\mathcal{N} \in V_{K,K',\sigma}$. Then Eq.(\ref{PLeq}) is valid. 
    
    We denote $$\mathcal{S}_{\mathcal{N}_{0}} := \bigcup\limits_{K \subset I, K' \subset J, \sigma} \mathcal{S}_{K,K',\sigma}$$ which is evidently a semialgebraic set. By the above discussion Eq.(\ref{PLeq}) is valid if and only if $\mathcal{N} \in \mathcal{S}_{\mathcal{N}_{0}}$.
    
    Now, remark that $\mathcal{S}_{\mathcal{N}_{0}}$ is a semialgebraic set defined by polynomials whose variables are entries of the matrices $W^{(k)}_{+}, W^{(k)}_{-}, b^{(k)}, t^{(k)}$, for each layer $k$. By lemma \ref{bridge} we may conclude that the defining polynomials for $\mathcal{S}_{\mathcal{N}_{0}}$ have in fact the entries of $W^{(k)}, t^{(k)}$ and $b^{(k)}$ as variables.
\end{proof}

\iffalse
---------------------------------------------------------------------------
We obtain a canonical stratification, indexed by the power set of $S$, of the space of coefficients as follows.

\begin{theorem}
Let $\mathcal{S}$ denote the space of coefficients of neural networks equivalent to a given neural network $\mathcal{N}_{0}$. We have a decomposition 
\begin{equation*}
    \mathcal{S} = \bigcup\limits_{T \in 2^{S}} \mathcal{S}_{T}
\end{equation*}
where $\mathcal{S}_T$ is the semialgebraic relation defined by the relevance of $T$ and redundancy of the complement of $T$ in equation \ref{plrelation}.  
\end{theorem}
\begin{proof}
For a given set of coefficients $W$, we always have a set of relevant indices $T$ in $S$. Therefore $\mathcal{S} \subset  \bigcup\limits_{T \in 2^{S}} \mathcal{S}_{T}$. (Elaborate later)
\end{proof}

\begin{Corollary}
The space $\mathcal{S}$ of neural networks equivalent to a given neural network $\mathcal{N}_{0}$ is a semialgebraic set. 
\end{Corollary}
----------------------------------------------------------------------------
\fi

\section{Conclusion and Further Study}

In this article we defined a neural network on a given neural architecture as the set of coefficients, i.e., weight matrices, bias and threshold vectors with real number entries. In other words a neural network for us is a tuple of matrices and vectors, essentially a point in $\mathbb{R}^N$ for some appropriate $N$. We introduced a natural equivalence relation between two neural networks on a given neural architecture by saying that they are equivalent if their associated functions are identically equal. We were able to show that the equivalence classes of neural networks with a fixed neural architecture are given by semialgebraic sets. 

One natural computational question is to find a set of defining polynomials for $\mathcal{S}_{\mathcal{N}_0}$ in the main theorem. We believe that a modified version of \emph{cylindrical algebraic decomposition} can be used to compute such polynomials. By construction the semialgebraic set obtained in the end only depends on the coefficients of $f_0$ and the network architecture. We may use quantifier elimination algorithms already existing in computer algebra  for a given neural architecture.  We intend to pursue this in the near future. 

Another direction of study is how $\mathcal{S}_{\mathcal{N}_0}$ varies in families, i.e., if we allow the coefficients of $f_{0}$ to vary then the $\mathcal{S}_{\mathcal{N}_0}$ vary as families of semialgebraic sets parameterized by coefficients of $f_0$. One invariant we may compute is the dimension of $\mathcal{S}_T$ (and see how it varies in families), although computing analogues of Hilbert polynomials for semialgebraic sets would be more general.

\bibliographystyle{unsrt}
\bibliography{researchRefs}

\end{document}